\documentclass[letterpaper]{article} 
\usepackage{aaai2026}  
\usepackage{times}  
\usepackage{helvet}  
\usepackage{courier}  
\usepackage[hyphens]{url}  
\usepackage{graphicx} 
\usepackage{natbib}  
\usepackage{caption} 
\usepackage{algorithm}
\usepackage{algorithmic}

\usepackage{soul}

\usepackage{url}
\usepackage{graphicx}
\usepackage{enumitem}
\usepackage{xspace}
\usepackage{comment}
\usepackage{multirow}
\usepackage{booktabs}
\usepackage{amsmath}
\usepackage{amssymb}
\usepackage{bm}
\usepackage{thmtools}
\usepackage{amsthm}
\usepackage{makecell}
\usepackage{colortbl}
\usepackage[table,xcdraw]{xcolor}
\definecolor{dustygreen}{RGB}{102, 153, 102}
\usepackage{dsfont}

\newcommand{\hide}[1]{}

\newcommand{\vswd}{\vspace{0.3em}}
\newcommand{\bit}{\vswd\begin{itemize*}}
\newcommand{\eit}{\end{itemize*}\vswd}
\newcommand{\ben}{\vswd\begin{enumerate*}}
\newcommand{\een}{\end{enumerate*}\vswd}
\newcommand{\bea}{\vspace{-0.0em}\begin{eqnarray}}
\newcommand{\eea}{\end{eqnarray}\vspace{-0.0em}}
\newcommand{\beq}{\vspace{-0.0em}\begin{equation}}
\newcommand{\eeq}{\end{equation}\vspace{-0.0em}}

\renewcommand{\bit}{\vswd\begin{compactitem}}
\renewcommand{\eit}{\end{compactitem}\vswd}
\renewcommand{\ben}{\vswd\begin{compactenum}}
\renewcommand{\een}{\end{compactenum}\vswd}

\newcommand{\method}{\textsc{ExPath}\xspace}
\newcommand{\classifier}{\textsc{PathMamba}\xspace}
\newcommand{\explainer}{\textsc{PathExplainer}\xspace}

\newtheorem{lemma}{Lemma}
\newtheorem{theorem}{Theorem}

\usepackage{newfloat}
\usepackage{listings}
\DeclareCaptionStyle{ruled}{labelfont=normalfont,labelsep=colon,strut=off} 
\floatstyle{ruled}
\newfloat{listing}{tb}{lst}{}
\floatname{listing}{Listing}
\title{Targeted Pathway Inference for Biological Knowledge Bases\\ via Graph Learning and Explanation}

\author {
    Rikuto Kotoge\textsuperscript{\rm 1},
    Ziwei Yang\textsuperscript{\rm 1,\rm 2},
    Zheng Chen\textsuperscript{\rm 1},
    Yushun Dong\textsuperscript{\rm 3},\\
    Yasuko Matsubara\textsuperscript{\rm 1},
    Jimeng Sun\textsuperscript{\rm 4},
    Yasushi Sakurai\textsuperscript{\rm 1}
}
\affiliations {
    \textsuperscript{\rm 1}SANKEN, The University of Osaka, Japan\\
    \textsuperscript{\rm 2}Bioinformatics Center, Institute for Chemical Research, Kyoto University, Japan\\
    \textsuperscript{\rm 3}Department of Computer Science, Florida State University, USA\\
    \textsuperscript{\rm 4}Department of Computer Science, University of Illinois Urbana-Champaign, USA\\
    \{rikuto88, chenz, yasuko,yasushi\}@sanken.osaka-u.ac.jp, \\
    yang.ziwei.37j@st.kyoto-u.ac.jp,
    yd24f@fsu.edu,
    jimeng@illinois.edu
}

\usepackage{bibentry}

\begin{document}

\maketitle

\begin{abstract}
  Retrieving targeted pathways in biological knowledge bases, particularly when incorporating wet-lab experimental data, remains a challenging task and often requires downstream analyses and specialized expertise.
In this paper, we frame this challenge as a solvable graph learning and explaining task and propose a novel subgraph inference framework, \method, that explicitly integrates experimental data to classify various graphs (bio-networks) in biological databases.
The links (representing pathways) that contribute more to classification can be considered as targeted pathways.
Our framework can seamlessly integrate biological foundation models to encode the experimental molecular data.
We propose ML-oriented biological evaluations and a new metric.
The experiments involving 301 bio-networks evaluations demonstrate that pathways inferred by \method are biologically meaningful, achieving up to 4.5× higher Fidelity+ 
(necessity) and 14× lower Fidelity- 
(sufficiency) than explainer baselines, while preserving signaling chains up to 4× longer.
\end{abstract}

\section{Introduction}
    \label{sec:intro}
    Decades of research have revealed that systems, from cells to organisms, can be considered biological networks \cite{Network_3}.
These networks have been compiled into public knowledge bases such as KEGG \cite{keggmapper} and STRING \cite{DataSTRING}, which document molecular (e.g., among genes or proteins) interactions and their roles in cellular functions.
While knowledge bases are continuously updated, a primary concern remains: \textit{they lack specificity for experimental data}.
The main objective of biological knowledge bases is to cover all possible interactions in a system. 
These networks are general and static.
In contrast, experimental studies focus on one specific condition or dataset, where only a subset of the network is actually relevant.
Our objective is to identify which interactions are active, meaningful, or target-specific in the given data, as shown in Figure \ref{fig:story_fig}.
In this paper, we propose to infer the bio-networks that capture targeted interactions from experimental data, thereby facilitating downstream analyses.

Many researchers have formulated this bio-network inference as a graph learning problem.
In this setting, interactions in a bio-network are modeled as graph edges, and experimental data are embedded as node features. 
Various computational and machine learning methods have been proposed to infer meaningful targeted graph structures.
Computational methods often rely on statistical node-centric metrics \cite{NACHER201657} to evaluate the importance of nodes.
Edges connected to highly ranked nodes are considered more important.
However, such objectives lack explicit inference of interactions and are computationally intractable for large bio-networks \cite{GeSubNet2024}.
Machine learning methods, particularly graph neural networks (GNNs), define explicit objectives such as link prediction or graph reconstruction, enabling direct inference of network structure \cite{GNNS2}.
Importantly, experimental data influence the learning process via node feature aggregation, making the inferred interactions more specific to the dataset.
However, existing methods are still in the early stages of exploration, and several key limitations remain unaddressed.

\begin{figure}[t]
\centering 
\includegraphics[width=\linewidth]{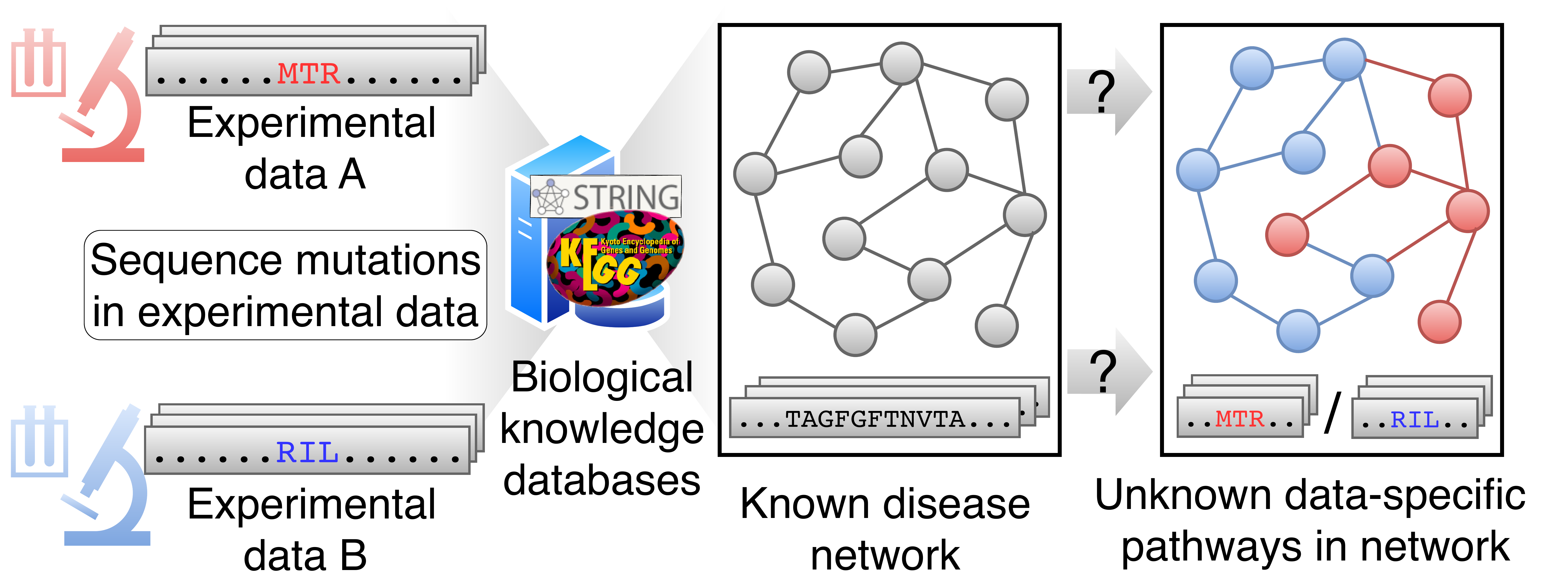}
\caption{This example illustrates two experimental datasets with different mutations (red and blue) that are mapped onto the same disease network, yet fail to reveal the distinct interactions that account for their differences.
}
\label{fig:story_fig}
\end{figure}

\begin{itemize}[left=0pt]
    \item \textbf{Implicit targeted interaction inference.}
    Their objectives aim to reconstruct the general graph structure accurately, including irrelevant interactions.
    Some works propose to gradually infer subgraph structure, weakening the influence of prior general bio-network information  \cite{graph1}.
    However, they still fail to explicitly identify the distinct interactions unique to different experimental data.
    \item \textbf{Lacking pathway modeling.} Existing works treat all interactions equally and independently, overlooking long-range dependencies in biological pathways.
    In reality, biological systems typically exhibit multi-step interactions, where one protein interaction triggers another, eventually leading to specific cellular outcomes.
    \item \textbf{Inadequate biologically plausible evaluation.}
    Existing methods typically require downstream biological analysis to qualitatively interpret the inferred interactions, which requires domain expertise.
    There is a lack of quantitative evaluation methods tailored for machine learning models.

\end{itemize}

To tackle the above challenges, we propose \method, a deep learning framework for inferring targeted data-specific pathways in bio-networks, with the following contributions:

\begin{itemize}[left=0pt]
\item \textbf{Graph explanation formulation for explicit interaction inference.}
we formulate bio-network inference as a subgraph learning and explanation task, and hence propose a graph-based model equipped with GNNExplainer.
Subgraphs, contributing most significantly to the learning objective, are explicitly identified as targeted interactions.

\item \textbf{Pathway-level encoding and explaining.}
To ensure these subgraphs capture high-order pathways, technically, we propose two novel models:
\classifier, a hybrid learning model, combines GNNs with state-space sequence modeling (Mamba) to learn both local interactions and global pathway-level dependencies;
\explainer
identifies objective-critical pathways  by learning novel pathway masks.
We also provide a theoretical analysis of \method's expressiveness and show that identified pathways capture higher-order structural patterns.

\item \textbf{A novel ML-oriented biological evaluation.}
We propose an evaluation workflow that directly incorporates model-derived subgraph importance scores to quantitatively assess their biological relevance.

\end{itemize}

\method can \textit{seamlessly integrate biological foundation models}, and in this work, we use the large protein language model, ESM-2 \cite{ESM}, as a case encoder.
We collect all available human pathway networks from KEGG \cite{KEGG2024}, resulting in 301 bio-network, using amino acid (AA) sequences as reference experimental data.
Extensive experiments demonstrate that the pathways inferred by \method are biologically meaningful, achieving up to 4.5× higher Fidelity+ 
(necessity) and 14× lower Fidelity- 
(sufficiency) than explainer baselines, while preserving signaling chains up to 4× longer.

\section{Related Work}
    \label{sec:related}
    Existing methods can be grouped into statistical topology-driven and data-driven deep graph learning methods.\\
\textbf{- Topology-driven Methods.} 
They use statistical metrics on structural properties of graphs, such as node degrees \cite{RSS,RSS1}, centrality \cite{MDS,MDS1,NACHER201657}, betweenness, or PageRank scores \cite{PMID:21149343,PPR1} to infer which substructures exhibit a more significant influence on the overall topology, thereby identifying more targeted interactions.\\
\textbf{- Deep Graph Learning Methods. } They incorporate experimental data during the learning process by embedding data as node representations.
They train GNNs with suitable objectives, such as link prediction or graph reconstruction \cite{GNNB2,GNNB1,GNNB3}, and the links that contribute most to these objectives can be considered the targeted interactions.
For instance, the works of  \cite{GraphSAGE, GCN1,GAT1} have been validated to predict protein functions within protein-protein interaction (PPI) networks.
Moreover, GNN models have been applied to incorporate RNA-Seq data, for tasks like predicting disease states and cell-cell relationships \cite{GNNS1, GNNS2}.\\
\textbf{Limitations of Previous Work.}
The topology-driven methods focus only on graph edges.
They cannot incorporate experimental data to infer biological networks.
While GNN-based methods can generate targeted interactions in a data-driven manner, their objectives do not explicitly focus on inferring networks and are typically task-specific. 
In contrast, our method focuses on directly explaining graph representations of bio-networks under specific experimental data. 
\section{Problem Formulation}
    \label{sec:preliminary}
    \begin{figure*}[t]
    \centering
    \includegraphics[width=\linewidth]{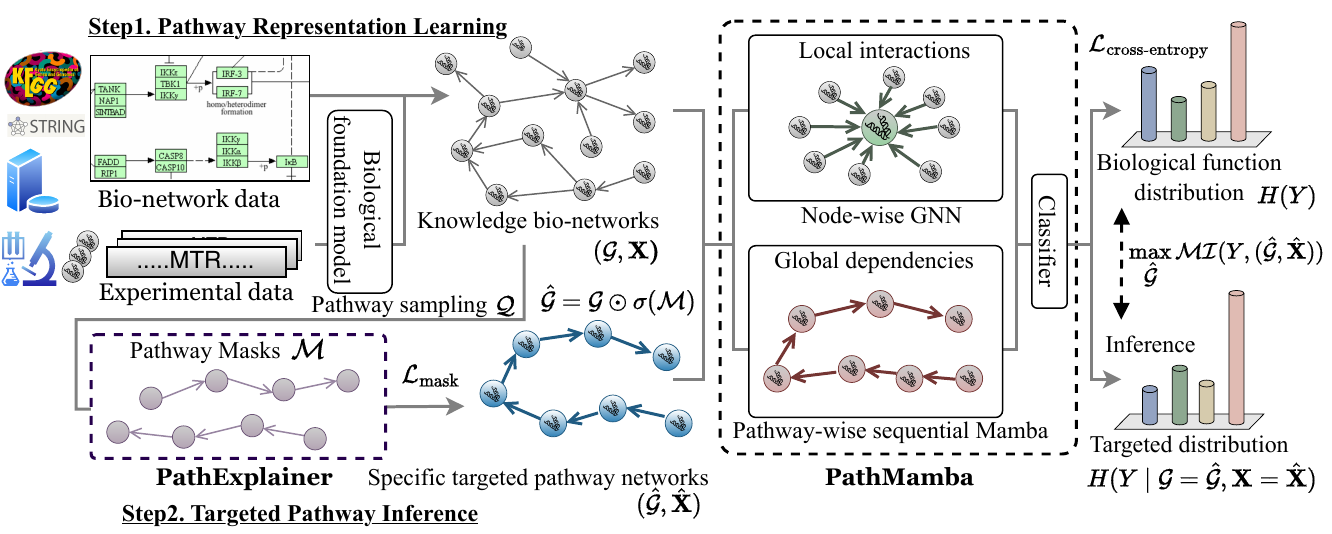}
    \caption{
    Overview of \method. Our method comprises two novel components. (1) \classifier combining graph neural networks with state-space sequence modeling (Mamba) to capture both local interactions and global pathway-level dependencies for pathway information learning; and (2) \explainer identifies functionally critical nodes and edges through trainable pathway masks for targeted pathway inference.}
    \label{fig:systemOV}
\end{figure*}

\noindent\textbf{Definition 1 (Knowledge bio-networks).}
The bio-networks can be represented as a graph $\mathcal{G} = (\mathcal{V}, \mathcal{E})$, where $\mathcal{V}$ denotes the vertices, each representing a molecule such as a gene or protein, and $\mathcal{E}$ is the set of edges, representing molecular interactions. 
Let $\mathbf{G} = {\{\mathcal{G}^{(m)}\}}_{m=1}^{M}$ denote a dataset comprising $M$ bio-networks. 
Each $\mathcal{G}^{(m)}$ is associated with a label $y^{(m)} \in \mathbf{Y}$, indicating its primary biologically functional class such as  {metabolism} or {human diseases}.

\noindent\textbf{Definition 2 (Molecular experimental data).}
For each node $v\!\in\!\mathcal{V}$, we are given a condition-specific feature vector
$\mathbf{x}_v\!\in\!\mathbb{R}^{d}$ derived from molecular experiments
(e.g., amino-acid sequence embeddings, gene-expression counts, or protein abundances
).  Collecting all nodes yields the matrix
$\mathbf{X}^{(m)}=\!\bigl[\mathbf{x}_v\bigr]_{v\in\mathcal{V}^{(m)}}$ for
network $\mathcal{G}^{(m)}$.

\noindent\textbf{Problem 1 (Bio-network inference).}
Let $\mathcal{G}=(\mathcal{V},\mathcal{E})$ be a curated graph and
 the node features $\mathbf{X}=\{\mathbf{x}_{v}\}_{v\in\mathcal{V}}$ obtained
from condition-specific \emph{experimental data}
(e.g., amino-acid embeddings).
Although $\mathcal{G}$ is static, the pair $(\mathcal{G},\mathbf{X})$
constitutes a \emph{data-specific graph} that reflects the molecular state
of the same pathway under the given experiment.
To this end, we formulate this problem as a \emph{two-stage sub-graph learning and explaining} task.
\begin{itemize}[leftmargin=*]
  \item \textbf{Task-1: Graph representation learning and classification.}\;
        Learn a classifier $F(\mathcal{G},\mathbf{X})$ that predicts the
        functional label $y\!\in\!\mathbf{Y}$ of an unseen data-specific
        graph.
  
  \item \textbf{Task-2: Targeted subgraph explanation.} Develop an explainer $E(\cdot)$ that identifies the
        smallest subgraph $\hat{\mathcal{G}}\!\subseteq\!\mathcal{G}$
        such that $F(\hat{\mathcal{G}},\mathbf{X})$ still outputs $y$.
\end{itemize}

\noindent\textbf{Problem 2 (Pathway modeling).}
Many biological functions arise from \emph{long, multi-step reaction pathways}
that span several hops in $\mathcal{G}$.
Capturing long-range dependencies is essential:
(1) for functional prediction, as perturbation effects often propagate across distant nodes; and
(2) for mechanistic insight into causal pathways beyond local interactions.
Different experimental conditions on the \emph{same} network, therefore
yield distinct, data-specific subgraphs $\hat{\mathcal{G}}$, each revealing
targeted pathway most responsible for the given data.
Hence,

\begin{itemize}[leftmargin=*]
  \item \textbf{Expectation: Pathway-level encoding and explaining.}\;
        The classifier $F(\cdot)$ and explainer $E(\cdot)$ leverage both graph topology and node features.
        $F(\cdot)$ aims to capture long-range dependencies, yielding high and class-balanced accuracy.
        Also, $E(\cdot)$ extracts subgraphs that retain biologically meaningful information of long pathways.

\end{itemize}

\section{Proposed Method}
    \label{sec:moethod}
    \subsection{Framework}

\method comprises two components: graph-based classification and post-hoc subgraph explanation, as shown in Figure \ref{fig:systemOV}. 
To tackle \underline{Task-1},
\classifier, a classifier combining GNNs with state-space sequence modeling, is to capture both local node-pair interactions and global pathway-level dependencies. 
To address \underline{Task-2}, \explainer, a graph explainer with pathway-wise masking, aims to identify the most influential subgraphs.
We explicitly integrate pathway information into both models to meet \underline{Expectation}.

Notably, our \method is compatible with large biological foundation models for encoding experimental data.
In this work, we leverage large protein language model encodings to investigate the mapping from amino acid (AA) sequences to corresponding pathway bio-networks.
Learning from AA sequence data is challenging due to its inherent complexity. 
Even slight variations can lead to significant structural changes, potentially disrupting protein functionality within pathways. 
Several studies focus on feature extraction in AA sequences, like AlphaFold \cite{AlphaFold}.

\noindent\textbf{Feedforward Process.} We first encode experimental data into node attributes using ESM-2 \cite{ESM}.
It is pre-trained on over 60 million AA sequences with parameter scaling up to 15 billion. 
We then train \classifier to learn pathway-level information and perform bio-network classification.
Finally, we apply \explainer to selectively highlight the minimal subgraphs that drive the final prediction, offering interpretable insights into key pathways.

\subsection{\classifier: Pathway Representation Learning}
\label{sec:classifier}

\classifier integrates the Graph Isomorphism Network (GIN) with a novel pathway-wise Mamba model.
It leverages the strengths of both global selective modeling mechanisms and message-passing GNNs. 
Specifically, inspired by GPS \cite{GPS}, our model avoids early-stage information loss that could arise from using GNNs in the initial layers. 
We employ novel pathway-wise global aggregation in efficient combination with random pathway sampling and sequential Mamba \cite{Mamba} modeling.  
At each layer, node and edge features are updated by aggregating the outputs of a pathway-wise Mamba as: 
\begin{eqnarray}
    X^{l+1}, &=& \texttt{PathMamba}^{l} \left( X^{l}, A \right), \\
    \textrm{computed as} \ \ \ 
    X^{l+1}_L, &=& \texttt{LocalGIN}^{l} \left( X^{l}, A \right), \\
    X^{l+1}_G, &=& \texttt{GlobalMamba}^{l} \left( X^{l}, A \right), \\
    X^{l+1}, &=& 
    \texttt{MLP}^{l}\left(X^{l+1}_L + X^{l+1}_G\right),
    \label{eqn:layer_equation}
\end{eqnarray}
where $A \in \mathbb{R}^{N \times N}$ is the adjacency matrix of a graph with $N$ nodes and $E$ edges; $X^{l} \in \mathbb{R}^{N \times d}$ represents the $d$-dimensional node features at layer $l$; $\texttt{LocalGIN}^{l}$ is a GIN; $\texttt{GlobalMamba}^{l}$ is a global pathway-wise aggregation layer; and $\texttt{MLP}^{l}$ is a two-layer multilayer perceptron (MLP) used to combine local and global features.

\subsubsection{Positional Encoding.} To address a fundamental limitation of GNNs \cite{GIN} or hybrid models \cite{GPS} to distinguish nodes with identical local structures, 
The node embedding \(\mathbf{h}_i \in \mathbb{R}^d\) and the positional encoding \(\mathbf{p}_i \in \mathbb{R}^K\) are concatenated and passed through a linear layer to obtain the final representation:
$\mathbf{x}_i = \text{Linear}([\mathbf{h}_i \| \mathbf{p}_i]),$
where \([\mathbf{h}_i \| \mathbf{p}_i] \in \mathbb{R}^{d + K}\) denotes the concatenation.

\subsubsection{Node-wise local aggregation.}
The GINs update Node features by aggregating information from their local neighbors. 
The GIN operation can be expressed as:
\begin{equation}
    X^{l+1}_L = \text{ReLU}\left( W^{l} \cdot \big( (1 + \epsilon) X^{l} + \sum_{j \in \mathcal{N}(i)} X^{l}_j \big) \right),
\end{equation}
where $\mathcal{N}(i)$ represents the set of neighbors of node $i$, $W^{l}$ is the learnable weight matrix at layer $l$, and $\epsilon$ is a trainable parameter controlling the importance of self-loops. 

\subsubsection{Pathway-wise global aggregation.}
\label{sub:global}

To capture long-term dependencies, we propose random pathway sampling and sequential pathway modeling in \classifier.

\noindent\emph{- Random Pathway Sampling. }
Formally, for each node $v_i$, we randomly sample a varied, single pathway with a maximum length of $L$. The sampling process is defined as:  
\begin{equation}
    \mathcal{Q} = \left\{ \bf{q}^i \mid \bf{q}^i \sim \text{Pathway}(v_i, L), \, |\bf{q}^i| \leq L \right\}_{i=1}^N,
\end{equation}  
where $N$ is the number of nodes in the graph, and $\bf{q}^i$ represents the sampled pathway for node $v_i$. Each pathway $\bf{q}^i$ is a sequence of nodes $\{v_i, v_{i_1}, v_{i_2}, \dots, v_{i_L}\}$, sampled according to a random walk process \cite{crawl}.  
The sampling process $\text{Pathway}(v_i, L)$ involves selecting a sequence of connected nodes starting from $v_i$.
The selection of each subsequent node is determined probabilistically, guided by the graph adjacency structure.

\noindent\emph{- Sequential Pathway Modeling. }
The forward propagation of the Mamba layer aggregates long-range dependencies along the sampled pathways. 
The selective sequential modeling of Mamba is well-suited for capturing such path information.
For each sampled pathway $\bf{q}^i \in \mathcal{Q}(X^{l})$, the Mamba layer processes the pathway sequentially as:
\begin{equation}
\begin{aligned}
\Delta_t &= \tau_\Delta(f_\Delta(\mathbf{x}_t^l)), \quad
\mathbf{B}_{t} = f_B(\mathbf{x}_t^l), \quad
\mathbf{C}_t = f_C(\mathbf{x}_t^l), \\
\mathbf{h}_t^l &= (1 - \Delta_t\cdot\mathbf{D}) \mathbf{h}_{t-1}^l + \Delta_t\cdot\mathbf{B}_t\mathbf{x}_t^l, \quad
X^{l+1}_G = C \cdot h^{l+1}_{L},
\end{aligned}
\label{eqn:mamba-updates}
\end{equation}
where $\mathbf{x}_t^l $ is the $t$-th input node feature matrix in pathway $\bf{q}^i$ at layer $l$.
$f_*$ are learnable projections and $\mathbf{h}_t^e$ is hidden state. $\tau_\Delta$ is the softplus function. 
The forgetting term 
$(1 - \Delta_t^e\cdot\mathbf{D})$
implements a selective mechanism analogous to synaptic decay or inhibitory processes that diminish outdated or irrelevant information. 
Conversely, the update term 
$\Delta_t^e\cdot\mathbf{B}_t^e$
mirrors gating that selectively reinforces and integrates salient new information. 
The projection $\mathbf{C}_t^e$ translates the internal state into observable outputs.
By processing each sampled pathway individually, the Mamba layer effectively aggregates information along each pathway. 
The aggregated pathway representations are then combined to form the updated node features $X^{l+1}_G$ for the next layer.

Afterward, we apply max pooling over the node features, i.e., $\{h_{v_i}\}_{i=1}^N$, followed by an MLP and softmax activation for the classification task.

\subsection{\explainer: Targeted Pathway Inference}\label{sec:explainer}
\explainer directly infers subgraphs to generate targeted pathways by leveraging the interpretability of \classifier. 
Vallina GNNexplainers \cite{gnnexplainer, pgexplainer}, which focus primarily on the node or edge level, often struggle to capture the global structures at the pathway level. 
In contrast, \explainer introduces \textbf{novel pathway mask training}, where entire pathways (i.e., connected nodes and edges) are selectively masked during training to evaluate their contributions to \classifier.

\subsubsection{Theoretical Objective.}
\explainer formalizes the identification of important subgraphs as an optimization problem. 
For a given graph \( \mathcal{G}\) and its features \( \mathbf{X} \), the explanation is defined as \( (\hat{\mathcal{G}}, \hat{\mathbf{X}}) \), where \( \hat{\mathcal{G}} \subseteq \mathcal{G} \) is the subgraph and \( \hat{\mathbf{X}} \) represents the selected features. 
The explanation is derived by optimizing the mutual information $\mathcal{MI}(\cdot)$ between the subgraph and the model's prediction, aiming to identify \( \hat{\mathcal{G}} \) that captures the predictive rationale of \classifier:
\begin{equation}
\max_{\hat{\mathcal{G}}} \mathcal{MI}(Y, (\hat{\mathcal{G}}, \hat{\mathbf{X}})) = H(Y) - H(Y \mid \mathcal{G} = \hat{\mathcal{G}}, \mathbf{X} = \hat{\mathbf{X}}),
\end{equation}
where \( H(Y) \) is the entropy of the predictions $Y$ and \( H(Y \mid \mathcal{G} = \hat{\mathcal{G}}, \mathbf{X} = \hat{\mathbf{X}}) \) is the conditional entropy given the explanation.
A lower conditional entropy indicates a more faithful and informative representation of the prediction.

\subsubsection{Optimization Framework.}
The optimization is approached by learning a pathway mask \( \mathcal{M} \) for the sampled pathway's edges and nodes based on random pathways \( \mathcal{Q} \) as described in Section~\ref{sub:global}. 
For each node \( v_i \), a random pathway \( q_i \) of length up to \( L \) is sampled. These pathways are then used to restrict the mask learning process within the sampled pathways, ensuring that the learnable pathway mask \( \mathcal{M} \) focuses on them. 
Specifically, the targeted subgraph \( \hat{\mathcal{G}}  \) is inferred based on  \( \mathbf{M} \) as: \( \hat{\mathcal{G}} = \mathcal{G} \odot \sigma(\mathcal{M}) \), where \( \sigma \) denotes the sigmoid function. 
The loss function for \explainer combines two components: a cross-entropy term for prediction consistency and regularization terms for sparsity:
\begin{equation}
\resizebox{1.0\linewidth}{!}{
    $\begin{aligned}
    \mathcal{L}_\text{mask} := -\sum_{c=1}^{C} \mathds{1}[y = c] \log P_\Phi(Y = y \mid \mathcal{G} = \hat{\mathcal{G}}, \mathbf{X} = \hat{\mathbf{X}} ) 
    + \lambda \|\mathcal{M}\|
    \end{aligned}$
},
\end{equation}
where \( \|M\| \) encourages sparsity in the edge selection, $Y$ is a random variable representing labels \{1, 2, \ldots, C\}, and $\lambda$ balances the trade-off between the prediction consistency and the sparsity regularization.
Hence, the identified important subgraphs and node features that contribute most to specific bio-networks are considered as targeted pathways.

\begin{table*}[t]
\caption{Baseline comparison results on bio-network classification. The best and second-best results are highlighted in \textbf{bold} and \underline{underline}, respectively. The gray-shaded rows indicate \classifier with different ESM-2 (encoder) parameter settings. }
\vspace{-0.4em} 
\label{tab:classification_result}
\begin{center}
\resizebox{\textwidth}{!}{
\begin{tabular}{l|cc|cc|cc|cc|c}
\toprule
       & \multicolumn{2}{c|}{Human Diseases} & \multicolumn{2}{c|}{Metabolism} & \multicolumn{2}{c|}{Organismal Systems} & \multicolumn{2}{c|}{Molecular \& Cellular Processes} &   Overall      \\ 
Methods         & Precision   & Recall                & Precision   & Recall            & Precision   & Recall                    & Precision   & Recall                                  & Accuracy     \\ 
\midrule
GCN               & 0.632 ± 0.013 & 0.669 ± 0.022           & 0.895 ± 0.009 & 0.958 ± 0.007       & 0.644 ± 0.037 & 0.630 ± 0.023               & 0.570 ± 0.033 & 0.357 ± 0.025                             & 0.683 ± 0.056  \\
GraphSAGE         & 0.583 ± 0.020 & 0.633 ± 0.072           & 0.890 ± 0.007 & 0.959 ± 0.014       & 0.553 ± 0.041 & 0.575 ± 0.031               & 0.526 ± 0.059 & 0.337 ± 0.062                             & 0.632 ± 0.037  \\
GAT               & 0.630 ± 0.015 & 0.643 ± 0.036           & \textbf{0.932 ± 0.017} & \underline{0.970 ± 0.008}       & \underline{0.659 ± 0.015} & \textbf{0.703 ± 0.010}               & 0.560 ± 0.058 & 0.370 ± 0.025                             & 0.690 ± 0.018  \\
GIN               & 0.688 ± 0.023 & 0.697 ± 0.014           & 0.912 ± 0.016 & 0.944 ± 0.022       & 0.629 ± 0.025 & 0.638 ± 0.041               & 0.606 ± 0.032 & \underline{0.497 ± 0.027}                             & 0.717 ± 0.013  \\
GPS & \underline{0.744 ± 0.018} & \underline{0.729 ± 0.024}           & 0.893 ± 0.006 & 0.955 ± 0.014       & 0.634 ± 0.026 & 0.658 ± 0.011               & 0.629 ± 0.060 & \textbf{0.507 ± 0.019}                             & \underline{0.726 ± 0.014}  \\
Graph-Mamba        & 0.707 ± 0.024 & 0.712 ± 0.024           & 0.897 ± 0.009 & 0.967 ± 0.007       & 0.626 ± 0.021 & 0.663 ± 0.033               & \textbf{0.700 ± 0.021} & 0.463 ± 0.032                             & 0.723 ± 0.014  \\
\midrule
\rowcolor{gray!10}
\classifier         & \textbf{0.786 ± 0.029} & \textbf{0.800 ± 0.033}           & \underline{0.915 ± 0.011} & \textbf{0.972 ± 0.005}       & \textbf{0.670 ± 0.026} & \textbf{0.703 ± 0.010}               & \underline{0.667 ± 0.035} & \underline{0.497 ± 0.028}                             & \textbf{0.744 ± 0.015}  \\
\rowcolor{gray!10} w/ 3B         & 0.752 ± 0.022 & 0.726 ± 0.027           & 0.917 ± 0.008 & 0.973 ± 0.010       & 0.661 ± 0.017 & 0.663 ± 0.023               & 0.656 ± 0.032 & 0.550 ± 0.042                             & 0.742 ± 0.009  \\
\rowcolor{gray!10} w/ 150M         & 0.764 ± 0.031 & 0.764 ± 0.011           & 0.906 ± 0.011 & 0.975 ± 0.013       & 0.639 ± 0.023 & 0.688 ± 0.025               & 0.653 ± 0.029 & 0.510 ± 0.030                             & 0.728 ± 0.013  \\
\rowcolor{gray!10} w/ 35M         & 0.748 ± 0.033 & 0.751 ± 0.019           & 0.914 ± 0.005 & 0.969 ± 0.007       & 0.634 ± 0.028 & 0.663 ± 0.028               & 0.633 ± 0.055 & 0.510 ± 0.049                             & 0.722 ± 0.013  \\
\rowcolor{gray!10} w/o ESM-2         & 0.380 ± 0.008 & 0.585 ± 0.015           & 0.669 ± 0.015 & 0.585 ± 0.015       & 0.241 ± 0.007 & 0.063 ± 0.019               & 0.378 ± 0.030 & 0.377 ± 0.043                             & 0.440 ± 0.010  \\
\bottomrule
\end{tabular}
}
\end{center}
\end{table*}

\subsection{Theoretical Analysis for Targeted Pathway Fidelity}
In this section, we place our method within the Weisfeiler–Lehman (WL) hierarchy to characterize its expressive power (Details can be found in Appendix \ref{appendix:tools}).
By proving that our explainer goes beyond the 1-WL limitation, we ensure that the extracted pathways capture higher-order structural patterns, establishing a theoretical upper bound on fidelity and inference
,supporting the empirical results.

\begin{restatable}{lemma}{exprforexpl}(Expressiveness for explanations).
\label{lmm:expr_for_expl}
When combined with higher expressive models
(\emph{e.g.}, it distinguishes more graphs), 
\explainer can generate more finely differentiated (and potentially 
more ``faithful'') explanation pathways (subgraphs). 
In contrast, a less expressive models
merges different graphs into larger equivalence classes, 
leading to non-unique, less granular explanations.
\end{restatable}

In Appendix \ref{app:proof}, we prove this by showing that 
the expressiveness of the underlying a graph classifier
$f$ determines the granularity of equivalence classes, with more expressive models enabling finer distinctions between graphs.

\begin{lemma}[Comparison with k-WL test]
\label{lmm:k-wl}
For every $k\ge 1$ there are graphs that are distinguishable
    by \classifier, but not
    by $k$-WL (and hence not by $k$-WL GNNs).
\end{lemma}

\begin{proof}
The proof of this theorem directly comes from the recent work \cite{crawl, graphmamba_kdd}. They prove a similar
theorem using 1-d CNNs \cite{crawl} or SSM \cite{graphmamba_kdd} with randomly sampled subgraphs.
Since our method adopts Mamba (an SSM architecture) combined with the random sampling strategy, their theoretical results are directly applicable to our setting.
\end{proof}

\begin{lemma}[Comparison with 1-WL test]
\label{lmm:1-wl}
\classifier is strictly more expressive than 1-WL GNNs.
\end{lemma}
\begin{proof}
We first note that \classifier\ contains the GIN as a sub-module, which has the same expressive power as the 1-WL test \cite{GIN}.
Therefore, \classifier\ is at least as expressive as 1-WL GNNs.
By Lemma~\ref{lmm:k-wl}, there are graphs that cannot be
distinguished by $1$-WL GNNs, but can be distinguished by
\classifier. Consequently, \classifier\ is strictly more expressive than 1-WL
GNNs.
\end{proof}

\begin{theorem}[Explanations of \method]
\label{thm}
Based on Lemma \ref{lmm:expr_for_expl}, \ref{lmm:k-wl}, and \ref{lmm:1-wl}, \method can generate more finely differentiated (and potentially 
more ``faithful'') explanation pathways (subgraphs) than 1-WL GNN-based methods, and not bounded by any WL GNN methods.
\end{theorem}

\section{Experiments and Results}
    \label{sec:exp}

\noindent\textbf{Dataset and Preprocessing. } 
We collected all available human pathway networks from the widely used knowledge database, KEGG \cite{kanehisa2000kegg}.
Our dataset consists of four main classes: Human Diseases, Metabolism, Molecular and Cellular Processes, and Organismal Systems, covering 301 bio-networks.
For nodes, we ensured that all protein nodes in the network were linked to their reference AA sequence data. 
The detailed data description and preprocessing can be found in Appendix~\ref{app:dataprecessing} and Table \ref{tb:data_sum}.

\noindent\textbf{Experimental Setup.} 
We conducted 10-fold stratified K-Fold cross-validation repeated five times. 
The optimal hyperparameters were determined using grid search.
Training for all models was implemented on NVIDIA A6000 GPU and Xeon Gold 6258R CPU. 

\begin{table}[t]
\centering
\caption{The computational efficiency comparison with hybrid models, including both training and inference runtime. }
\resizebox{0.99\linewidth}{!}{%
\label{tab:computational}
\begin{tabular}{lcc}
\toprule
\textbf{Methods} & \textbf{Training Time (msec)} & \textbf{Inference Time (msec)} \\
\midrule
GPS              &  29.2 $\pm$ 2.3             &   10.3 $\pm$ 0.3  \\
Graph-Mamba               &  34.8 $\pm$ 0.4            & 9.5 $\pm$ 0.2                               \\
\classifier       &  \textbf{24.4 $\pm$ 0.9}          &  \textbf{6.9 $\pm$ 0.2}       \\                     
\bottomrule
\end{tabular}%
}
\vspace{-0.2cm}
\end{table}

\subsection{Experiment-I: Pathway Learning}
\label{subsec:exp1}

\textbf{Objective.} This experiment aims to evaluate whether \method can classify diverse bio-networks and benchmark its performance against baseline models.

\noindent\textbf{Baselines and Metrics.}
We collected baselines from both message-passing GNNs and more advanced graph models, including GCN \cite{GCN}, GraphSAGE \cite{GraphSAGE}, GAT \cite{GAT}, GIN \cite{GIN}, GPS \cite{GPS}, and Graph-Mamba \cite{Geraph-Mamba}.
The detailed introduction of these baselines and the selection motivation can be found in Appendix \ref{app:gnnbaseline}.
We employed precision, recall, and overall accuracy for the performance evaluation. We used 650M ESM-2 for \classifier and all baselines as the node feature encoding model.

\noindent\textbf{Results.}
Table~\ref{tab:classification_result} demonstrates that \classifier achieves the highest accuracy (0.744), outperforming all GNNs, GPS (0.726), GraphMamba (0.723). 
Furthermore, it secures best or second-best positions across all functional categories, demonstrating its robust ability to generalize across diverse pathway structures.
The gray-shaded rows indicate the results of removing ESM-2 and modifying the model size in terms of F1 scores.
When ESM-2 is removed, the accuracy decreases significantly (0.74 → 0.44).
The results highlight the importance of AA-seq and the limitations of prior studies that were unable to leverage this information. 

Table \ref{tab:computational} compares the training and inference times of our model with other expressive hybrid models, using a batch size of 32. Our training time is 30\% faster than GPS, and inference time is 27\% faster than Graph-Mamba (complexity analysis can be found in Appendix \ref{app:complexity}).

\begin{figure}[t]
    \centering
    \includegraphics[width=0.88\linewidth]{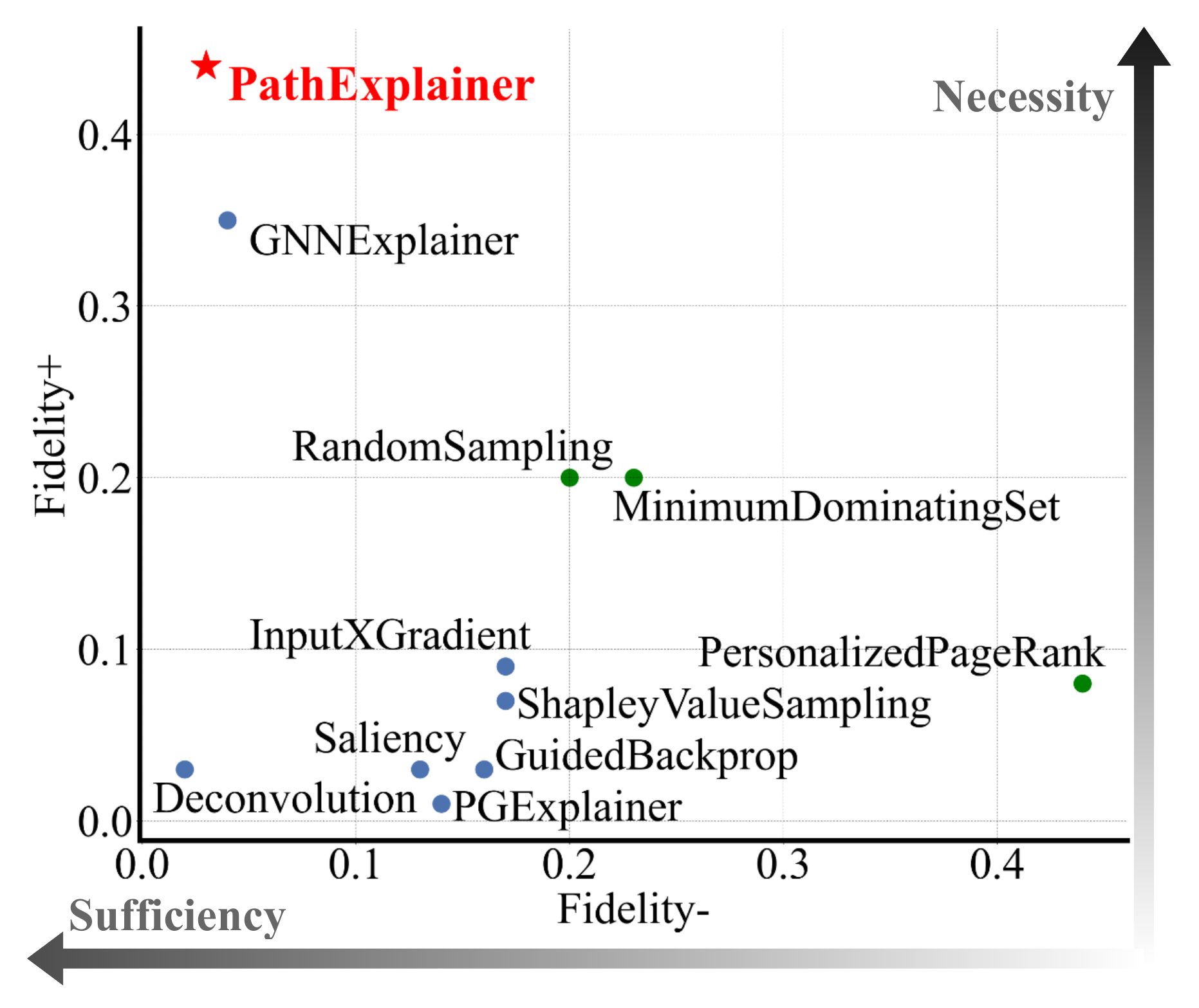} 
    \vspace{-0.4em} 
    \caption{Fidelity+ (necessity $\uparrow$) and Fidelity- (sufficiency $\downarrow$) scores of extracted subgraphs.
    Our \explainer achieves the best performance on both metrics.}
    \vspace{-0.4em}
    \label{fig:scatter}
    \vspace{-0.4em}
\end{figure}

\subsection{Experiment-II: Pathway Inference}
\label{subsec:exp2}

\textbf{Objective. } This experiment aims to quantify the fidelity of extracted subgraphs using \explainer and validate the importance of pathways specific to biological functions.

\noindent\textbf{Baselines.}
We collected baselines from three categories (i) conventional statistical methods: Random Sampling \cite{10.1093/bioinformatics/bth163}, PersonalizedPageRank \cite{PMID:21149343}, and MinimumDominatingSet \cite{NACHER201657, doi:10.1073/pnas.1311231111}); (ii) gradient-based methods: Saliency \cite{saliency}, InputXGradient \cite{InputXGradient}, Deconvolution \cite{deconvolustion}, ShapleyValueSampling \cite{shapley}, and GuidedBackpropagation \cite{guidedbackpropagation}; (iii) GNN-specific explainer method: GNNExplainer \cite{gnnexplainer} and PGExplainer \cite{pgexplainer}.
All details can be found in Appendix \ref{app:gnnbaseline}. 

\noindent\textbf{Metrics. }
We evaluated the distinctiveness of the pathways inferred by \explainer using fidelity metrics, \emph{Fidelity+} and \emph{Fidelity-}.
Fidelity+ measures how well the selected features support accurate predictions, while Fidelity- checks how much accuracy drops when only those features are kept.
We further evaluated the length of pathways.  
\emph{Max Path Length} captures the longest simple path in each subgraph, reflecting whether long signaling chains are retained.
\emph{Average Diameter} measures the typical node-to-node distance, showing how spread out the nodes remain after extraction.
All details can be found in Appendix \ref{app:gnnbaseline} and \ref{metric}.

\begin{table}[t]
\centering
\caption{Comparison of pathway-preservation ability across subgraph extraction baselines. Higher path length and diameter indicate better retention of long-range interactions.}
\resizebox{0.99\linewidth}{!}{%
\label{tab:length}
\begin{tabular}{lcc}
\toprule
\textbf{Methods} & \textbf{Max Path Length} & \textbf{Average Diameter} \\
\midrule
Minimum Dominating Set              &  6          &   2.00  \\
Random Sampling              &  11          &   2.90  \\
Personalized Page Rank    &  4          &   1.53  \\
\classifier-GNNExplainer                 &  9     & 2.90                              \\
GPS-\explainer      &  12         &  3.95      \\               
\textbf{\method} (Ours)       &  \textbf{16}         &  \textbf{4.20}      \\                     
\bottomrule
\end{tabular}%
}
\vspace{-0.2cm}
\end{table}

\noindent\textbf{Results.} 
Figure \ref{fig:scatter} shows that \explainer achieves the highest fidelity+ and the lowest fidelity-. 
The main reason is that GNNExplainers optimise a mask for each individual node or edge level, whereas \explainer infers pathways (subgraphs) as a single coherent unit. 
Deconvolution simply aggregates all edges with large gradients, it almost covers every active edge.
This  pushes low fidelity$^{-}$, but since it retains many redundant edges, removing them hardly changes the output, so \emph{fidelity$^{+}$} (necessity) stays low.
GNN-specific or gradient-based methods (\textcolor{blue}{blue points}) show lower fidelity- compared to traditional methods (\textcolor{dustygreen}{green points}), indicating that the learned AA-seq enables the identification of sufficient subgraphs. 

Table~\ref{tab:length} presents that our method attains  up to 4 ×  longer preserved paths and up to 2.7 × larger diameters than competing approaches.
This supports that the identified sufficient and necessary features capture biologically meaningful pathways and meets our \underline{Expectation}.

\subsection{Experiment-III: Biological Meaningfulness}
\textbf{Objective. }
We propose an evaluation workflow to analyze the biological significance of the subgraphs and pathways extracted from our method. 
This workflow should integrate the \emph{weighting/ranking scores of pathway inferred by \method} into biological metrics, enabling the direct quantification of outputs from the models. 

\noindent\textbf{Proposed Evaluation Metrics. }
We designed experiments centered on Gene Ontology (GO) analysis \cite{GO}, focusing on the nodes within the extracted subgraphs. 
The results provide a list of GO terms highlighting the biological functions most significantly represented in the input gene (corresponding to protein) nodes \cite{GO}.
Then we proposed Number of Enriched Biological Functions (\textbf{\#EBF}) and Enrichment Contribution Score (\textbf{ECS}) to evaluate breadth and depth of the extracted functions \cite{GeSubNet2024}.
A higher \#EBF indicates broader functional diversity within the subgraph.
ECS evaluates the relative contribution of the top-weighted genes.
The details and definitions of metrics can be found in Appendix \ref{metric}.

\begin{table}
\centering
\caption{Biological meaningfulness comparison results: The best-performing results are highlighted in \textbf{bold}. The second-best results are highlighted in \underline{underline}. }
\resizebox{0.90\linewidth}{!}{%
\label{table: bio_result}
\begin{tabular}{lccc} 
\toprule
                        Methods & \#EBF ($\uparrow$)          & ECS ($\uparrow$)           & P-value ($\downarrow$)         \\ 
\midrule
RSS                      & 5.29           & 0.27          & 0.045           \\
MDS                      & 6.34           & 0.23          & 0.043           \\
PPR                      & 6.64           & 0.23          & 0.042           \\ 
\midrule
GIN-GNNE                 & 6.94           & 0.59          & 0.041           \\
GPS-GNNE                 & 8.88           & 0.22          & 0.039           \\
GraphMamba-GNNE              & 10.73          & 0.21          & 0.042           \\
PathMamba-GNNE           & \underline{11.89}  & \underline{0.73}  & \textbf{0.036}  \\ 
\midrule
GIN-PathE                & 11.06          & 0.69          & 0.041           \\
GPS-PathE                & 8.26           & 0.43          & \underline{0.037}   \\
GraphMamba-PathE             & 10.89          & 0.59          & 0.038           \\
\textbf{\method} & \textbf{14.77} & \textbf{0.84} & \textbf{0.036}  \\
\bottomrule
\end{tabular}
}
\end{table}

\begin{figure}[t]
\centering
\includegraphics[width=0.91\linewidth]{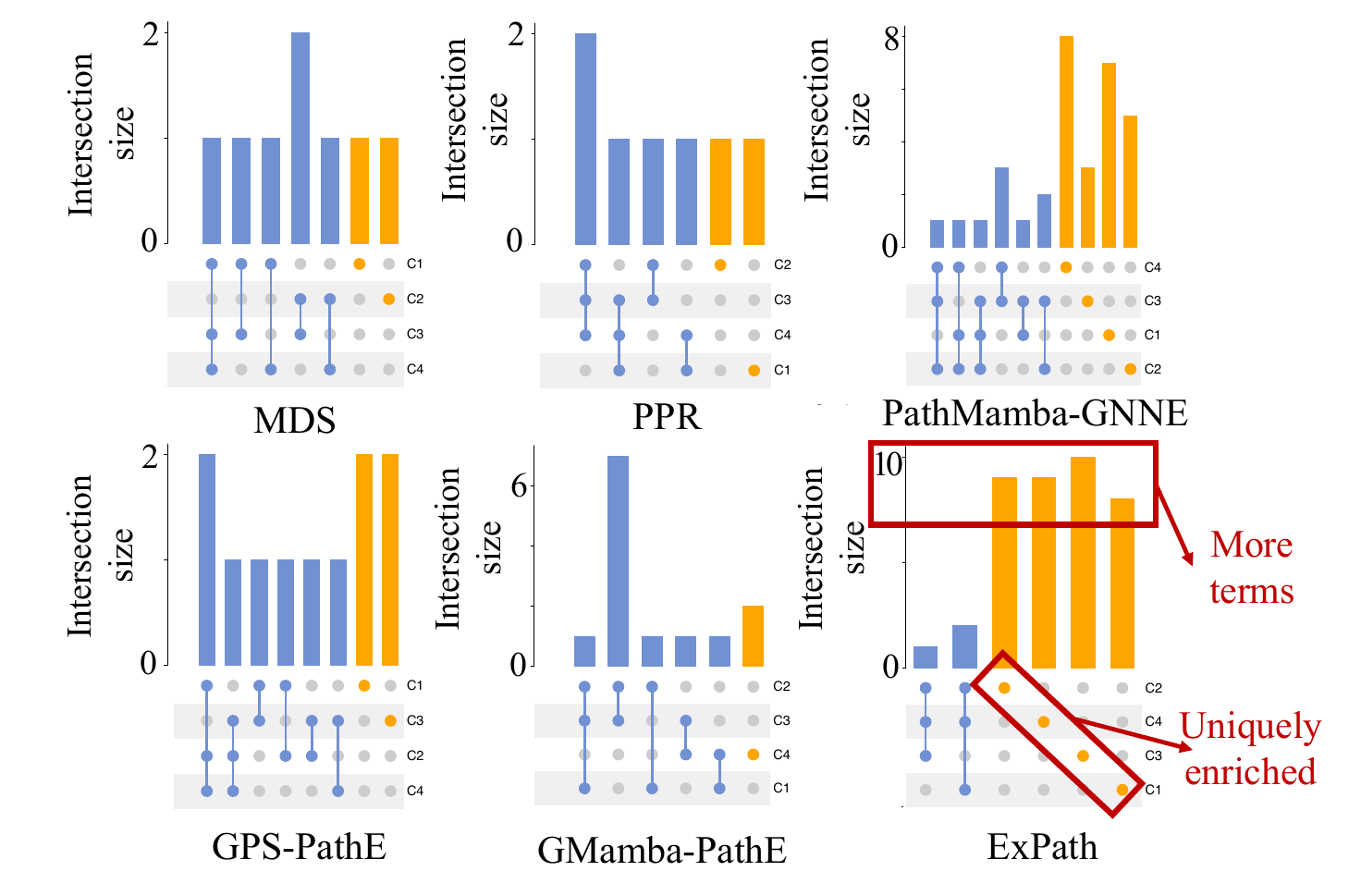}
\caption{UpSet plot of enriched GO terms across four pathway classes, based on top feature sets from subgraphs for different methods. Orange indicates GO terms uniquely enriched in one class, and blue represents GO terms shared across multiple classes. MDS, and PPR stand for Minimum Dominating Set, and Personalized PageRank, respectively. 
}
\label{fig:go}
\end{figure}

\begin{figure}[t]
\centering
\includegraphics[width=0.8\linewidth]{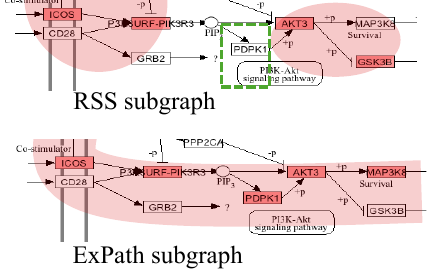}
\caption{Comparison of subgraphs extracted from the TCR signaling pathway using two different methods. TCR Subgraph A is from the RSS method, and TCR Subgraph B is from the proposed method. The subgraph nodes and their signaling modules are colored in red. The disruptions within signaling paths are marked in green boxes.
A full pathway graph can be found in Appendix H.
}
\label{fig:path}
\end{figure}

\noindent\textbf{Results. }
Table~\ref{table: bio_result} presents the biological meaningfulness comparison results for subgraphs extracted using different methods. 
Overall, PathMamba-PathE achieves the highest performance across \#EBF, ECS, and P-value. 
This highlights its ability to extract biologically relevant structures within pathway networks, effectively balancing breadth and depth.
While conventional methods (RSS, MDS, and PPR) perform relatively poorly in overall \#EBF and ECS, with almost boundary P-values achieved.

Figure \ref{fig:go} evaluates the differences in enriched GO terms across four pathway classes based on top gene sets from subgraphs extracted by different methods.
The upset plot reveals that PathMamba-PathE identifies the most extensive sets of unique GO terms (shown as the orange bars and links) across all four pathway classes while maintaining fewer shared terms (shown as the blue bars and links) among different classes.
This suggests that PathMamba-PathE tends to assign appropriate weights to genes based on their importance within the network and effectively captures the distinct biological roles of top-ranked genes in specific pathway classes.

\subsection{Experiment IV: PoC of Biological Case Study}

\noindent\textbf{Setup.}  
We make a Proof-of-Concept case analysis using the T cell receptor (TCR) signaling pathway, which is a well-characterized human pathway.
In this case study, we compare subgraphs extracted by two methods: TCR Subgraph A, generated using the RSS method, and Subgraph B, obtained via our proposed method. 
Each method selects the top 10\% highest-ranked nodes and their associated edges to construct a representative subgraph.
The detailed information and analysis of the case study can be found in Appendix~\ref{app:casestudy}.

\noindent\textbf{Results. }  
In Figure~\ref{fig:path}, the upper subfigure presents the TCR subgraph extracted by the RSS method, while the lower one corresponds to our \method. 
In subgraph A, generated by the RSS method, high scores are distributed uniformly across a broad range of nodes within the TCR pathway. 
However, this suggests unnatural, fragmented signal propagation, as evidenced by the numerous isolated red-marked nodes and a green broken connection in PDPK1.
In contrast, subgraph B, extracted by our method, exhibits a strong focus on the PI3K-AKT signaling axis~\cite{PI3K} and the downstream components of the MAP3K8 survival~\cite{nfkb}, as highlighted by a coherent red-marked path.

\noindent\textbf{Discussion. }  
In summary, the extracted subgraphs by \method align with the needs of real-world pathway analysis practices: maintaining signal continuity within regulatory cascades and even accommodating relatively long signaling paths, making them more suitable for focused analyses of bio-network regulatory mechanisms.

\section{Conclusion}
    \label{sec:conclusion}
    We introduced \method, a novel framework for understanding targeted pathways within biological knowledge bases. \method integrates \classifier, a hybrid model to capture local and global dependencies; and \explainer, a subgraph learning module that identifies key nodes and edges via trainable pathway masks. 
\method seamlessly integrated biological foundation models to encode the experimental molecular data.
We also introduced machine-learning-oriented biological evaluations and a new metric.
The experiments involving 301 bio-networks evaluations demonstrated that pathways inferred by \method maintain biological meaningfulness.
Future work will expand \method to analyze other types of bio-networks, enabling broader applications in systems biology and medicine.
\section{Acknowledgments}
    The authors would like to thank the reviewers.
This work was supported by
JST BOOST, Japan Grant Number JPMJBS2402, “Program for Leading Graduate Schools” of The University of Osaka, 
JSPS KAKENHI Grant-in-Aid for Scientific Research Number JP24K20778, 
NSF award SCH-2205289, SCH-2014438, and IIS-2034479, 
JST CREST JPMJCR23M3,
JST START JPMJST2553,
JST CREST JPMJCR20C6,
JST K Program JPMJKP25Y6,
JST COI-NEXT JPMJPF2009,
JST COI-NEXT JPMJPF2115,
the Future Social Value Co-Creation Project - The University of Osaka.
    

\bibliography{BIB/main}

\clearpage

\setcounter{secnumdepth}{2} 

\appendix
 \section*{Appendix}
     \label{sec:appendix}
     

\begin{table}[h]
    \centering
    \caption{Summary of pathway data across four pathway classes.}
\resizebox{0.95\linewidth}{!}{%
    \begin{tabular}{llcccc}
        \toprule
        \multicolumn{2}{l}{Pathway class}& \#Samples & \#Nodes & \#Edges & AA-seq Length \\
        \midrule
        C1&Human Diseases & 83 & 40 & 42 & 583 \\
        C2&Metabolism & 78 & 16 & 42 & 511 \\
        C3&\begin{tabular}[c]{@{}l@{}}Molecular and \\cellular processes\end{tabular} & 80 & 30 & 37 & 636 \\
        C4&Organismal systems & 60 & 44 & 49 & 638 \\
        \midrule
        \multicolumn{2}{l}{Overall}  & 301 & 32 & 42 & 590 \\
        \bottomrule
    \end{tabular}
    }
    \label{tb:data_sum}
\end{table}

\section{Graph isomorphism and WL test}
\label{appendix:tools}
\textbf{Graph isomorphism} refers to the problem of determining whether two graphs are structurally identical, meaning there exists a one-to-one correspondence between their nodes and edges. This is a crucial challenge in graph classification tasks, where the goal is to assign labels to entire graphs based on their structures. A model that can effectively differentiate non-isomorphic graphs is said to have high expressiveness, which is essential for accurate classification. In many cases, graph classification models like GNNs rely on graph isomorphism tests to ensure that structurally distinct graphs receive different embeddings, which improves the model’s ability to correctly classify graphs. 

\noindent\textbf{Weisfeiler-Lehman (WL) test} is a widely used graph isomorphism test that forms the foundation of many GNNs. In the 1-WL framework, each node's representation is iteratively updated by aggregating information from its neighboring nodes, followed by a hashing process to capture the structural patterns of the graph. GNNs leveraging this concept, such as Graph Convolutional Networks (GCNs) and Graph Attention Networks (GATs), essentially perform a similar neighborhood aggregation, making them as expressive as the 1-WL test in distinguishing non-isomorphic graphs \citep{GIN}.  
Modern GNN architectures adhere to this paradigm, making the 1-WL a standard baseline for GNN expressivity.

\section{Proof}
\label{app:proof}
\exprforexpl*
\begin{proof}
Given equivalent graphs $\mathcal{G}_1$ and $\mathcal{G}_2$, let $f : \mathcal{G} \to \mathbb{R}^k$ be a GNN-based model, and let $\sim$ denote the equivalence relation induced by $f$, i.e.,
\[
\mathcal{G}_1 \sim \mathcal{G}_2 \quad \Longleftrightarrow \quad f(\mathcal{G}_1) = f(\mathcal{G}_2).
\]
We define an \emph{explanatory subgraph}
$\hat{\mathcal{G}}\subseteq\mathcal{G}$ and consider a \classifier objective given by

\[
\mathcal{L}\bigl(\hat{\mathcal{G}};f,\mathcal{G}\bigr)
   \;=\;
   \alpha\,D\bigl(f(\mathcal{G}),f(\hat{\mathcal{G}})\bigr)
   \;+\;
   \beta\,\Omega\bigl(\hat{\mathcal{G}}\bigr),
\]
where $D(\cdot,\cdot)$ is a divergence between model outputs and
$\Omega(\cdot)$ penalises the size or complexity of~$\hat{\mathcal{G}}$.
The optimal explanation for $\mathcal{G}$ is
\[
E(\mathcal{G})
   \;=\;
   \operatorname*{arg\,min}_{\hat{\mathcal{G}}\subseteq\mathcal{G}}
   \mathcal{L}\!\bigl(\hat{\mathcal{G}};f,\mathcal{G}\bigr).
\]
If $\mathcal{G}_1 \sim \mathcal{G}_2$, then any pair of optimal explanations $E(\mathcal{G}_1)$ and $E(\mathcal{G}_2)$ 
must yield the same minimum objective value. Consequently, there is no unique 
explanation across $\mathcal{G}_1$ and $H$ within the same equivalence class.
Since $\mathcal{G}_1 \sim \mathcal{G}_2$, we have $f(\mathcal{G}_1) = f(\mathcal{G}_2)$. 
By definition of $E(\mathcal{G}_1)$,
\[
\mathcal{L}\bigl(E(\mathcal{G}_1); f, \mathcal{G}_1\bigr) \;=\;
\min_{\hat{\mathcal{G}}_1 \subseteq \mathcal{G}_1} \; \mathcal{L}(\hat{\mathcal{G}}_1; f, \mathcal{G}_1).
\]
Similarly, for $\mathcal{G}_2$,
\[
\mathcal{L}\bigl(E(\mathcal{G}_2); f, \mathcal{G}_2\bigr) \;=\;
\min_{\hat{\mathcal{G}}_2 \subseteq \mathcal{G}_2} \; \mathcal{L}(\hat{\mathcal{G}}_2; f, \mathcal{G}_2).
\]
Since $f(\mathcal{G}_1) = f(\mathcal{G}_2)$, the divergence term 
$D\bigl(f(\mathcal{G}_1),f(\hat{\mathcal{G}_1})\bigr)$ behaves the same as 
$D\bigl(f(\mathcal{G}_2),f(\hat{\mathcal{G}_2})\bigr)$ for corresponding substructures $\hat{\mathcal{G}_1}$ and  $\hat{\mathcal{G}_2}$.
Thus, for any $\mathcal{G}_1 \sim \mathcal{G}_2$, 
\[
\mathcal{L}\bigl(E(\mathcal{G}_1); f, \mathcal{G}_1\bigr) = \mathcal{L}\bigl(E(\mathcal{G}_2); f, \mathcal{G}_2\bigr),
\]
which implies the optimal explanations are not uniquely determined 
beyond the equivalence class $[\mathcal{G}_1]$ (the set of all graphs equivalent to $\mathcal{G}_1$). 
In other words, \explainer 
cannot uniquely distinguish between subgraphs $\hat{\mathcal{G}_1}$ and $\hat{\mathcal{G}_2}$ such that $\mathcal{G}_1 \sim \mathcal{G}_2$. 
Less expressive models 
merge different graphs into larger equivalence classes, 
leading to non-unique, less granular explanations.
In contrast, 
when combined with higher expressive models (i.e. \classifier), 
\explainer can generate more finely differentiated explanation subgraphs.
\qedhere
\end{proof}

\section{Computational Complexity Analysis}
\label{app:complexity}
Given \( K \) tokens, the complexity of Mamba \cite{Mamba} is  linear with respect to \( K \). 
For \( m \geq 1 \), for each node \( v \in V \), we generate \( |V| \) sampled pathways with length \( m \), the time complexity of global module mamba
would be: 
\[ O(|V| \times m), \]
since we have \( O(|V| \times m) \) tokens. 
Our \classifier is faster than the quadratic time complexity \( O(|V| ^2) \) of graph transformers \cite{GPS}.

In practice, combined with GNN, which requires \( O(|V| + |E|) \) time,
the total complexity would be:
\[
O(|V| + |E|),
\]
dominated by the GNN complexity,
since $m$ represents only a subset of pathways sampled from the total possible nodes \( V \) 
(\(m \ll |V|\)).

\section{Preprocessing}
\label{app:dataprecessing}
\paragraph{Overview.} The KEGG database is a comprehensive resource that integrates genomic, chemical, and systemic functional information, providing curated pathway networks for various biological processes derived from experimental data and expert annotations.
For human pathways, KEGG offers detailed representations of processes such as metabolism, genetic information processing, environmental information processing, and human diseases.
Each category includes pathways organized into networks, where nodes represent biological entities—such as genes, proteins, enzymes, or metabolites—and edges denote their interactions.
These interactions encompass direct biochemical reactions, regulatory relationships, and signaling pathways that govern cellular mechanisms, ultimately forming pathways related to various functional biological processes.   
\paragraph{Data Acquisition.} We searched for and downloaded all the raw data for the human pathway network (referred to as Homo Sapiens in most bio-databases) using KEGG APIs.
The data underwent a series of preprocessing steps to ensure its quality and relevance.
\paragraph{(1) Node Filtering and Feature Construction.} We ensured that all protein nodes in the network were linked to well-characterized genetic origins, specifically their reference amino acid sequence data \cite{keggseq}. 
Using KEGG's gene-to-protein mapping, we filtered the dataset to retain only protein nodes with associated genomic annotations \cite{keggmapper}. 
Protein nodes lacking sequencing data or genetic associations in KEGG were excluded to reduce noise caused by incomplete or ambiguous sequence information. 
For the retained nodes, their amino acid sequences were extracted and utilized as input features, ensuring a biologically meaningful representation for the learning task.

\paragraph{(2) Edge and Structure Cleaning.} We streamlined the network structure by removing redundant or biologically insignificant interactions. 
Specifically, we eliminated non-functional self-loops (edges connecting a node to itself without annotated biological relevance) and isolated nodes lacking any edges. 
This process included both nodes that were initially isolated and those rendered isolated following the first step of node filtering. 
Since these elements no longer contributed to the network's connectivity or functional variation, their removal reduced unnecessary complexity and ensured the network focused exclusively on meaningful and biologically interpretable interactions \cite{clean}.

\paragraph{(3) Graph Conversion} We preserved the edge with properties of protein-protein interactions while removing directional information to transform the network into an undirected graph. 
This conversion enabled the analysis to emphasize undirected, pairwise interactions, which are often more pertinent to network-based studies, such as clustering, community detection, or functional enrichment analyses.

\paragraph{(4) Functional Categorization.} The pathways were organized into functional classes based on their KEGG pathway labels, ensuring that biologically related pathways were grouped together. 
The original labels for Environmental Information Processing and Genetic Information Processing were combined into a unified class, Molecular and Cellular Processes, to reflect their shared biological roles in cellular signaling, communication, and gene regulation \cite{kegglabel}. 
As a result, the pathway data used in this study was categorized into four main classes: Human Diseases, Metabolism, Molecular and Cellular Processes, and Organismal Systems.

\section{Baselines}
\label{app:gnnbaseline}

\subsection{Classifier models.}
We collected baselines from both message-passing GNNs and more advanced hybrid graph models. 
\begin{itemize}[left=0pt]
    \item \textbf{Message-passing GNNs}: 
Graph Convolutional Network (GCN) \cite{GCN} serves as a foundational GNN, leveraging spectral graph theory for node feature aggregation. 
GraphSAGE \cite{GraphSAGE} improves scalability by introducing neighborhood sampling and learnable aggregation functions. 
Graph Attention Network (GAT) \cite{GAT} incorporates attention mechanisms to assign different importance to neighbors during feature aggregation. 
Graph Isomorphism Network (GIN) \cite{GIN} achieves high expressivity, distinguishing graph structures with a focus on injective neighborhood aggregation. 
\item \textbf{Hybrid graph models}: 
GPS \cite{GPS} combines GNNs with transformer-style global attention to effectively process both local and global graph structures. 
Similarly, Graph-Mamba \cite{Geraph-Mamba} processes local structures using GNNs and leverages the Mamba module to capture global node relationships.
\end{itemize}

\subsection{Explainer methods.} 
We collected baselines from statistical methods, gradient-based methods, and GNN-specific explainer methods. 
\begin{itemize}[left=0pt]
    \item \textbf{Statistical methods}: 
    Random Sampling (RRS) \cite{10.1093/bioinformatics/bth163} serves as a simple baseline by selecting nodes or edges randomly for comparison. 
Personalized PageRank (PPR) \cite{PMID:21149343} computes node importance by incorporating a teleportation mechanism that biases the random walk towards specific nodes, effectively capturing both local and global graph structures. 
Minimum Dominating Set (MDS) \cite{NACHER201657, doi:10.1073/pnas.1311231111} identifies a minimal set of nodes that can collectively influence or dominate all other nodes in the graph, providing insights into critical nodes for coverage or control.
Notably these three statistical methods only do not use AA sequence node features.
\item \textbf{Gradient-based methods}: 
Saliency \cite{saliency} highlights features based on the magnitude of input gradients. 
InputXGradient \cite{InputXGradient} combines input features with their gradients to capture feature significance. 
Deconvolution \cite{deconvolustion} focuses on reconstructing important input features, emphasizing positive influences. 
ShapleyValueSampling \cite{shapley} estimates feature importance using a game-theoretic approach. 
GuidedBackpropagation \cite{guidedbackpropagation} refines gradients to highlight only relevant activations.
\item \textbf{GNN-specific explainability approaches}, we adopted GNNExplainer \cite{gnnexplainer} to uncover subgraphs and features that are critical for predictions, and PGExplainer \cite{pgexplainer}, which uses a neural network to identify significant graph components.
\end{itemize}

\section{Metrics}
\label{metric}

\begin{itemize}[left=0pt]
    \item \textbf{Fidelity+}: Fidelity+ measures how well the important features identified by the model contribute to accurate predictions. It is defined as:  
\[
\text{Fidelity+} = \frac{1}{N} \sum_{i=1}^{N} \big( f(G_i) - f(G_i \setminus S_i) \big),
\]
where \( f(G_i) \) is the prediction score for graph \( G_i \), and \( f(G_i \setminus S_i) \) is the prediction score after removing the subgraph \( S_i \) identified as important.
\item \textbf{Fidelity-}: Fidelity- evaluates the drop in prediction accuracy when the identified important features are retained while others are removed. It is defined as:  
\[
\text{Fidelity-} = \frac{1}{N} \sum_{i=1}^{N} \big( f(S_i) - f(G_i) \big),
\]
where \( f(S_i) \) is the prediction score for the retained subgraph \( S_i \), and \( f(G_i) \) is the original score.
\item \textbf{Diameter}: Diameter of a (connected, unweighted) graph \(G=(V,E)\) is  
\[
\operatorname{diam}(G)=\max_{u,v\in V}\; d_{G}(u,v),
\]
where \(d_{G}(u,v)\) denotes the length (number of edges) of the shortest path between vertices \(u\) and \(v\).  
Intuitively, it is the length of the \textbf{longest} among all \textbf{shortest} paths, capturing the farthest distance that must be traversed within the graph.  
Because each method yields many subgraphs, we report the average diameter:
\[
\overline{D}=\frac{1}{|\mathcal{S}|}\sum_{G_S\in\mathcal{S}}\operatorname{diam}(G_S),
\]
i.e.\ the mean diameter over the set \(\mathcal{S}\) of all extracted subgraphs.  
A larger \(\overline{D}\) indicates that the explanations retain longer end-to-end interactions, aligning with pathway-level preservation. 

\end{itemize}

As a classic method of biological functional enrichment analysis, GO analysis evaluates whether specific biological processes, molecular functions, or cellular components are statistically overrepresented in a given set of genes (i.e., gene nodes from subgraphs) compared to a background gene set \cite{enrichment}.
The results provide a list of GO terms that highlight the biological functions most significantly represented in the input gene nodes \cite{GO}.
In our study, we conducted GO enrichment analysis using the R package \texttt{clusterProfiler} \cite{clusterProfiler} to identify enriched GO terms.

\begin{itemize}[left=0pt]
    \item \textbf{\#EBF}: To assess \textbf{Breadth} of the subgraphs' biological functions, we use the Number of Enriched Biological Functions (\#EBF) as a metric. 
For an input subgraph \( S_i \), \#EBF is defined as:
\[
\#EBF(S_i) = |GO_{\text{enriched}}(S_i)|,
\]
where \( S_i \) denotes the i-th input subgraph, \( GO_{\text{enriched}}(S_i) \) is the set of significantly enriched GO terms associated with the genes in \( S_i \), and \( |GO_{\text{enriched}}(S_i)| \) is the size of the set of enriched GO terms for subgraph \( S_i \).
A high \#EBF value indicates broader functional diversity within the subgraph.

\item \textbf{ECS}: To assess \textbf{Depth} of the subgraphs' biological functions, we used the Enrichment Contribution Score (ECS) as a metric.  
The ECS evaluates the relative contribution of the top-weighted genes, denoted as \( G_{\text{Top}} \), to the enrichment of biological functions. 
The ECS can be assessed by following steps:
Let \( G = \{ g_1, g_2, \dots, g_n \} \) represent the ranked list of gene nodes, sorted by their importance scores (weights) \( w = \{ w_1, w_2, \dots, w_n \} \), where \( w_1 \ge w_2 \ge \dots \ge w_n \).  
Define \( G_{\text{Top}} = \{ g_1, g_2, \dots, g_{\text{Top}} \} \) to include only genes with the top weights, selected based on a ratio \( R\% \) (defaulted as 30\%), as a subset of \( G \).
Then, perform GO analysis based on \( G_{\text{Top}} \) for each input subgraph \( S_i \).
The ECS is calculated as the average number of enriched items for each gene in \( G_{\text{Top}} \) across all subgraphs \( S_i \), and is defined as:
\[
\text{ECS} = \frac{1}{P} \sum_{i=1}^{P} \frac{|GO_{\text{top-enriched}}(S_i)|}{|G_{\text{Top}}|},
\]
where \( P \) is the total number of tested subgraphs, \( GO_{\text{top-enriched}}(S_i) \) is the set of enriched GO terms for subgraph \( S_i \) based on \( G_{\text{Top}} \), and \( |G_{\text{Top}}| \) is the number of genes in the subset \( G_{\text{Top}} \).

\item \textbf{P-value}: To assess \textbf{Reliability} of the subgraphs' biological functions, we use the well-established statistical concept P-value as a metric. 
Specifically, the P-value is calculated as the average of the P-values from statistical tests performed for biological function enrichment in each subgraph above.
Since during the GO analysis, we accept the item only with a P-value lower than 0.05, the average P-value reported here is naturally lower than this threshold.
A lower average P-value indicates greater reliability in the enrichment results across the subgraphs.
\end{itemize}

\section{Ablation Study}

\begin{table}[h]
\centering
 \caption{Sequential Mamba module evaluation with LSTM. }
\resizebox{0.6\linewidth}{!}{%
\label{tab:architecture}
\begin{tabular}{llcc}
\toprule
Models & Overall Accuracy \\
\midrule
\classifier      & 0.744 ± 0.015   \\  
w/ LSTM            & 0.730 ± 0.014                              \\
                   
\bottomrule
\end{tabular}%
}
\end{table}

\begin{figure*}
\centering
\includegraphics[width=0.99\linewidth]{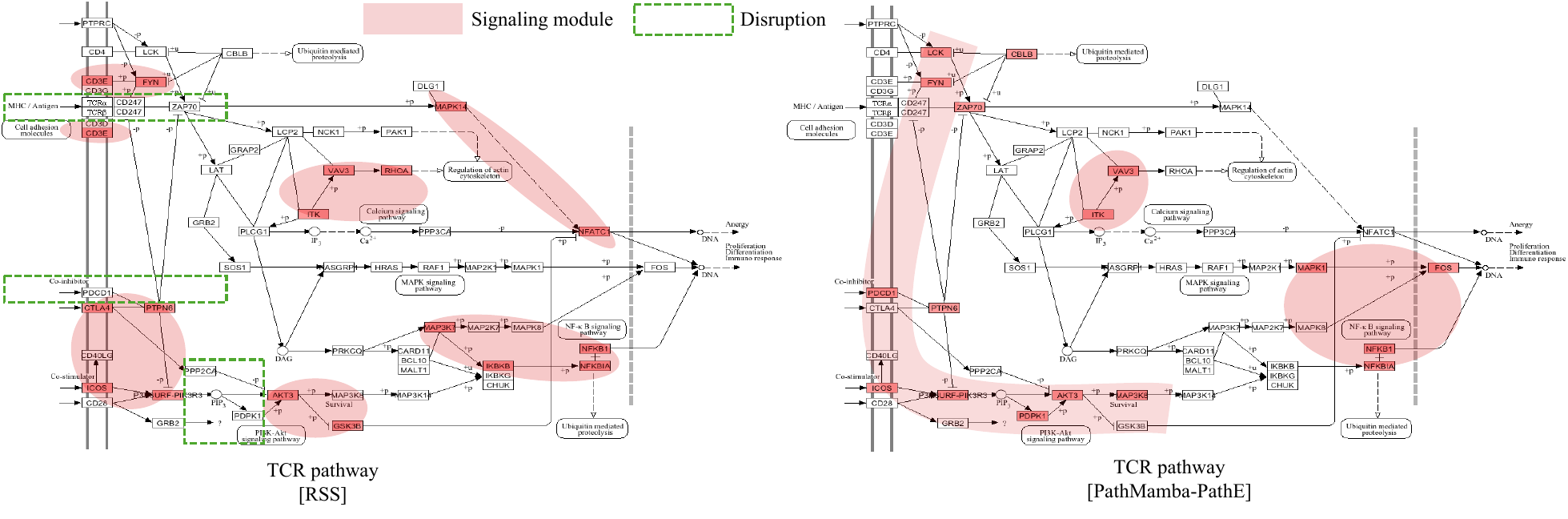}
\caption{Comparison of subgraphs extracted from the TCR signaling pathway using two different methods. The TCR Subgraph on the left is from the RSS method, and the TCR Subgraph on the right is from the proposed method. The subgraph nodes and their signaling modules are colored in red. The disruptions within signaling paths are marked in green boxes.}
\label{fig:path_all}
\end{figure*}

\subsection{Sequential module evaluation}
Table \ref{tab:architecture} shows the results when the Mamba module of \classifier was replaced with an RNN-based model (LSTM). The classification accuracy was 0.730, and \classifier (0.744) exceeds the LSTM-based model. This suggests that Mamba's selective state-space modeling is particularly important for identifying complex biological relationships in pathways where both local interactions and long-range functional dependencies influence biological outcomes.

\subsection{Classifier evaluation}
\begin{table}[t]
\centering
 \caption{\explainer fidelity score comparison with other classifier models. }
\resizebox{\linewidth}{!}{%
\label{tab:classifier}
\begin{tabular}{llcc}
\toprule
Methods &Fidelity+ & Fidelity- & (Accuracy) \\
\midrule
GIN             &  0.689 $\pm$ 0.012             &   0.390 $\pm$ 0.008   & (0.717 ± 0.013)\\
GPS            &  0.763 $\pm$ 0.017             &   0.529 $\pm$ 0.011    & (0.726 ± 0.014)\\
Graph-Mamba               &  0.708 $\pm$ 0.015            & 0.430 $\pm$ 0.012 & (0.723 ± 0.014)                              \\
\classifier       &  0.442 $\pm$ 0.012          &  0.037 $\pm$ 0.008    & (0.744 ± 0.015)   \\                     
\bottomrule
\end{tabular}%
}
\end{table}
Table \ref{tab:classifier} presents the results of \explainer when the classifier model is changed. In our \classifier, the fidelity- score is reduced to less than one-tenth, indicating that the extracted subgraph alone produces nearly identical results. 
The lower fidelity+ score also shows strong representational power of \classifier, as it can still perform reasonably well even when important subgraphs are removed. For our objective of identifying function-specific pathways, subgraph sufficiency is particularly important, which our method achieves effectively.

\section{Case Study}\label{app:casestudy}

\subsection{T cell receptor (TCR) signaling pathway}
The T cell receptor (TCR) signaling pathway is a cornerstone of adaptive immunity, orchestrating antigen-specific T cell activation, clonal expansion, and effector differentiation \cite{TCR_1}. 
This pathway is initiated upon engagement of the TCR complex with peptide-MHC ligands, triggering a cascade of intracellular signaling events mediated by the Src-family kinases LCK and FYN, leading to phosphorylation of the immunoreceptor tyrosine-based activation motifs (ITAMs) within the CD3 and $\zeta$-chain subunits \cite{TCR_3}. 
Subsequent recruitment and activation of ZAP-70 further amplify downstream signaling through the LAT signalosome, engaging multiple adaptor proteins and second messengers that regulate key pathways, including calcium mobilization, Ras-MAPK, and NF-$\kappa$B signaling, which collectively drives gene transcription, metabolic reprogramming, and cytoskeletal remodeling necessary for T cell function \cite{TCR_2}. 

Precise modulation of these signaling cascades is critical for maintaining immune homeostasis, as dysregulation is implicated in a spectrum of immune disorders, including autoimmunity, primary immunodeficiencies, and T cell malignancies, where aberrant activation or attenuation of TCR signaling disrupts immune tolerance, facilitates chronic inflammation, or drives oncogenic transformation. 
Analysis of the molecular intricacies of TCR signaling helps therapeutic interventions, including immune checkpoint modulation, CAR-T \cite{CART} cell engineering, and small-molecule inhibitors aimed at restoring immune balance and targeting immune-related diseases.

\subsection{Results and Discussion}
\textbf{TCR Subgraph A: The RSS Method. }  
As shown in Figure \ref{fig:path_all}, the subgraph on the left, generated by the RSS method, distributes high scores uniformly across a broad range of nodes within the TCR pathway. 
However, this hints at unnatural, fragmented signal propagation, as highlighted by numerous discrete red-marked nodes. 
The absence of coherent signaling continuity, as indicated by the disrupted green-boxed regions, suggests that the method fails to prioritize biologically meaningful regulatory modules. 
Since its broader coverage spans multiple branches of the pathway without emphasizing critical molecular hubs, it limits the utility in pinpointing key functional perturbations.

\paragraph{TCR Subgraph B: The Proposed Method. }  
In contrast, as shown in Figure \ref{fig:path_all}, the subgraph on the right extracted by our method exhibits a strong focus on the PI3K-AKT signaling axis \cite{PI3K} and the downstream components of the NF-\(\kappa\)B \cite{nfkb} pathway, as highlighted by the coherent red-marked path.
 These regions are believed to be crucial for regulating T cell survival, proliferation, and cytokine production \cite{nfkball}.  
Notably, the subgraph includes key regulatory genes such as \textit{MAPK1}, \textit{MAP3K8}, and \textit{NFKB1} \cite{MAPK1,NFKB1} within a compact set of prioritized nodes.  
The enrichment of these molecular hubs suggests that our method effectively captures biologically significant signaling elements, aligning with well-established immune regulatory mechanisms.  

In our case study, the proposed method provides a focused selection of key regulatory pathways, emphasizing PI3K-AKT and NF-\(\kappa\)B signaling and their downstream effectors, which are crucial for modulating immune responses. 
In contrast, the RSS method, while providing broader pathway coverage, lacks specificity, making it less suitable for pathway analyses requiring mechanistic interpretation.

\section{Implementation Details}
We performed a grid search over key hyperparameters to ensure optimal performance. Table \ref{tab:hyperparams} summarizes the values considered for each parameter. 

\begin{table}[ht]
\centering
\caption{Hyperparameter settings.}
\label{tab:hyperparams}
\begin{tabular}{ll}
\toprule
Parameter        & Value \\ \midrule
Batch size                                & 32 \\
Epochs                                    & 50 \\
Cross-validation folds                    & 10 \\
Learning rate                           & 0.001 \\
Weight decay                     & 5e-4 \\
Number of layers  & \{1, 2\} \\
Hidden size              & \{32, 64, 128\} \\
Walk length                               & \{4, 8, 16\} \\
\bottomrule
\end{tabular}
\end{table}

\end{document}